\documentclass[lettersize,journal]{IEEEtran}

\usepackage{microtype}
\usepackage{amsmath}
\usepackage{amsfonts}
\usepackage{amssymb}
\usepackage{mathtools}
\usepackage{amsthm}
\usepackage{graphicx}
\usepackage{booktabs}
\usepackage{multirow}
\usepackage[table]{xcolor}
\usepackage{array}
\usepackage[caption=false,font=normalsize,labelfont=sf,textfont=sf]{subfig}
\usepackage{textcomp}
\usepackage{stfloats}
\usepackage{url}
\usepackage{verbatim}
\usepackage{cite}
\usepackage{float}
\usepackage{hyperref}
\usepackage{algorithmic}
\usepackage{algorithm}

\theoremstyle{plain}
\newtheorem{theorem}{Theorem}[section]
\newtheorem{proposition}[theorem]{Proposition}

\theoremstyle{definition}
\newtheorem{definition}[theorem]{Definition}

\theoremstyle{remark}

\newcommand{\cmark}{\checkmark}
\newcommand{\xmark}{\times}

\begin{document}

\title{Rethinking Air-Ground Collaboration: A Progressive Cross-Task Benchmark and Socialized Learning Framework}

\author{Zhoupeng Guo, Yunqi Zhu, Zhihe Fan, Xinjie Yao,
        Ruipu Zhao, Boan Tao, Yiming Sun, Zhen Wang,
        and Pengfei Zhu%
\thanks{This work was supported in part by the National Natural Science Foundation of China under Grants 62436002, and 62506073, in part by Tianjin Natural Science Funds for Distinguished Young Scholar under Grant 23JCJQJC00270, and in part by the Xiongan New Area Science and Technology Innovation Special Project of the National Key Research and Development Program under Grant No.2025XAGG0039.}
\thanks{Zhoupeng Guo, Yiming Sun, and Pengfei Zhu are with the School of Automation, Southeast University, Nanjing 210096, China.}%
\thanks{Yunqi Zhu is with the School of Computer Science and Engineering, University of New South Wales, Sydney, NSW 2052, Australia.}%
\thanks{Zhihe Fan is with the School of Sports Training, Tianjin University of Sport, Tianjin 301617, China}%
\thanks{Xinjie Yao is with the Faculty of Information Engineering and Automation, Kunming University of Science and Technology, Kunming 650500, China.}%
\thanks{Ruipu Zhao, and Boan Tao are with the School of Artificial Intelligence, Tianjin University, Tianjin 300350, China.}%
\thanks{Zhen Wang is with the School of Artificial Intelligence, Hebei University of Technology, Tianjin 300401, China.}%
}%

\maketitle

\begin{abstract}
Air-ground collaborative perception is crucial for robust visual understanding in real-world dynamic environments. However, existing studies typically formulate collaboration as single-task cross-view fusion, overlooking the functional dependencies among localization, target association, and fine-grained parsing. In addition, the heterogeneous nature of aerial and ground views introduces substantial geometric, scale, and occlusion discrepancies, making uniform feature sharing vulnerable to negative transfer. To tackle these issues, we model air-ground perception as a progressive cross-task collaboration task and construct the Air-Ground Progressive Collaboration (AGPC) benchmark, a spatio-temporally aligned benchmark comprising more than 745K raw video frames. Built upon this benchmark, we propose Socialized Co-Perception (SCP), a coarse-to-fine framework that organizes collaboration progressively from aerial global localization to ground target association and identity-aware parsing. Its core module, the Dual-Layer Router (DLR), decouples input-side multi-scale expert selection from output-side task-conditioned modulation, enabling selective cross-view and cross-task interaction while suppressing harmful interference. Extensive experiments demonstrate the effectiveness of SCP. It achieves a 3.73\% coevolutionary gain and a 7.86\% improvement in average downstream performance. These results show that task-conditioned collaboration is more effective than uniform fusion for heterogeneous air-ground perception. The code is available at https://github.com/g1136639260-spec/AGSCP.


\end{abstract}


\section{Introduction}
\label{sec:Introduction}


Analogous to animal societies, where cooperation arises from differentiated roles and complementary capabilities rather than uniform behavior~\cite{I1social,I1social2}, effective air-ground collaboration in machine systems requires more than the joint optimization of shared representations~\cite{R15}. However, existing approaches typically treat sub-tasks in isolation, failing to explicitly exploit their functional complementarity. Furthermore, the assumption that shared feature learning suffices for cross-view synergy breaks down in heterogeneous environments. Discrepancies in imaging geometry, scale, and occlusion induce distribution shifts that amplify negative transfer and destabilize representation learning~\cite{PCGrad20,CAGrad21}. These observations indicate that the core limitation lies not in model capacity, but in the implicit assumption of cross-view and cross-task homogeneity. Taken together, the problem is shaped jointly by cross-view heterogeneity across air-ground platforms and
  cross-task dependency across perception objectives.


In machine learning, air-ground collaboration and multi-task learning are closely related yet built upon different assumptions, leading to a paradigm mismatch. Existing air-ground studies mainly formulate cross-view cooperation around task-specific objectives such as localization, retrieval, tracking, and cooperative detection, where one view serves as a task prior for the other, but the interaction remains largely stage-isolated~\cite{SAFA19,VIGOR21,TransGeo22}. By contrast, representative multi-task frameworks mostly focus on single-view dense prediction and emphasize shared representations, adaptive sharing, loss balancing, or conflict-aware optimization over spatially aligned tasks~\cite{PCGrad20,CAGrad21,TaskPrompter23}. Although both lines succeed in their respective settings, in cross-view and cross-task scenarios, severe discrepancies between aerial and ground observations weaken the implicit single-view alignment assumption behind most shared-representation designs, systematically compounding negative transfer and hindering transferable representation learning.

\begin{figure}[t] 
    \centering
    \includegraphics[width=\linewidth]{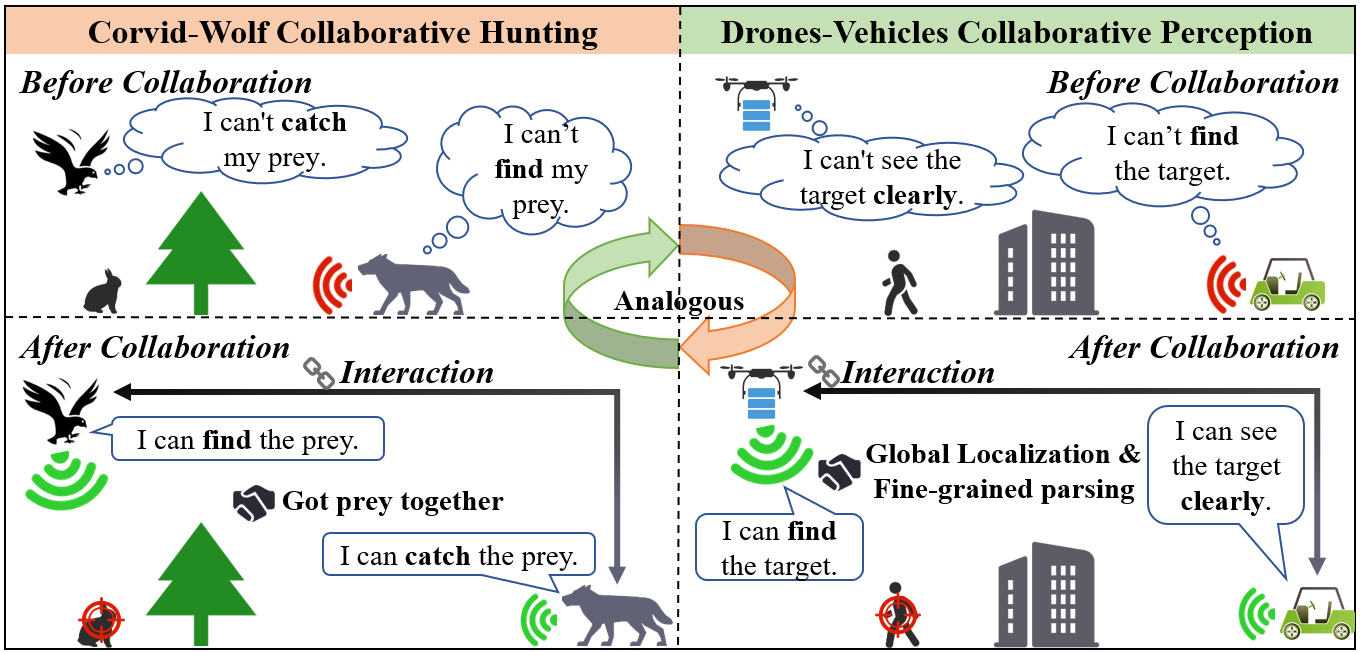}
    \caption{Social co-perception in animal and machine societies}
    \label{fig:inspired}
\end{figure}

To systematically study this joint cross-view and cross-task setting, we construct the Air-Ground Progressive Collaboration (AGPC) benchmark, which couples aerial and ground observations across localization, association, and parsing in a spatio-temporally aligned setting. Within this benchmark, however, existing paradigms remain limited by their underlying assumptions: air-ground methods typically emphasize cross-view cooperation around isolated objectives, whereas multi-task frameworks are largely developed for spatially aligned single-view settings. As a result, this setting exposes two key challenges for effective air-ground collaboration.

\begin{enumerate}
    \item[1:] \emph{How to eliminate task isolation for collaboration?}
    \item[2:] \emph{How to exploit cross-task air-ground complementarity?}
\end{enumerate}


Addressing the first challenge requires collaboration to be organized across task stages rather than remaining stage-isolated, whereas addressing the second requires a mechanism that can preserve complementary cues while reducing interference from heterogeneous views. To this end, we draw inspiration from sociobiology and ethology. As illustrated in Fig.~\ref{fig:inspired}, mutualism between ravens and wolves~\cite{I5} arises from asymmetric yet complementary roles, where aerial agents provide global localization and ground agents perform targeted capture. This mechanism highlights the importance of sociability~\cite{E1-dsc} in complex perception, where agents exchange conditional information to resolve mutual ambiguities. Guided by this insight, air-ground collaboration can be viewed as a coarse-to-fine inference process. In this paradigm, global cues guide cross-view interaction to support fine-grained task reasoning, while fine-grained expertise reciprocally reinforces preceding coarse-level representations through bi-directional coevolution. Such a formulation naturally enables progressive task coupling for alleviating task isolation and task-conditioned information exchange for exploiting cross-view complementarity.


To support the investigation of these problems, we first establish the AGPC benchmark, which comprises more than 745K raw video frames, 31,334 detection boxes, 4,319 ReID samples, and 4,268 segmentation instances. By organizing these tasks into a progressive evaluation chain from global localization to cross-view association and fine-grained parsing, AGPC provides an experimental basis for analyzing how cross-view cues propagate across tasks under realistic heterogeneity. Building on this benchmark, we develop Socialized co-Perception (SCP) as a concrete framework for progressive air-ground collaboration. Within this setting, SCP operationalizes sociability and complementarity through a coarse-to-fine perception pipeline and a Dual-Layer Router (DLR) that regulates information exchange to promote beneficial interactions while reducing interference. Our contributions are summarized as follows:

\begin{itemize}
\item We establish AGPC, an air-ground cross-task benchmark with over 745K real-world raw video frames, covering detection, ReID, and semantic segmentation, and enabling systematic evaluation of air-ground collaboration.
\item We propose the SCP framework, a socialized learning framework that realizes these principles through a coarse-to-fine design and a DLR module for cross-view and cross-task interaction. 
\item We provide an information-theoretic perspective on collaborative perception, identifying sociability and complementarity as key principles for air-ground collaboration. 
\end{itemize}

\begin{figure*}[t]
  \vskip 0.2in
  \centering
  \subfloat[]{\includegraphics[width=0.40\textwidth]{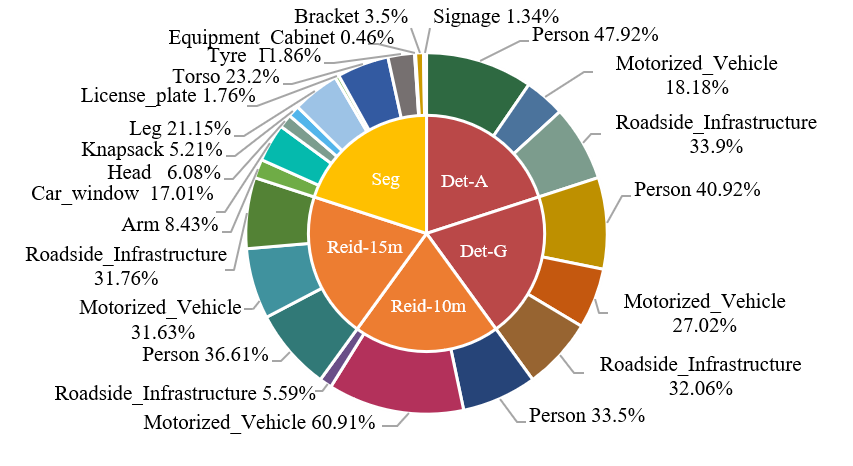}\label{fig:data_stat_a}}
  \hfill
  \subfloat[]{\includegraphics[width=0.58\textwidth]{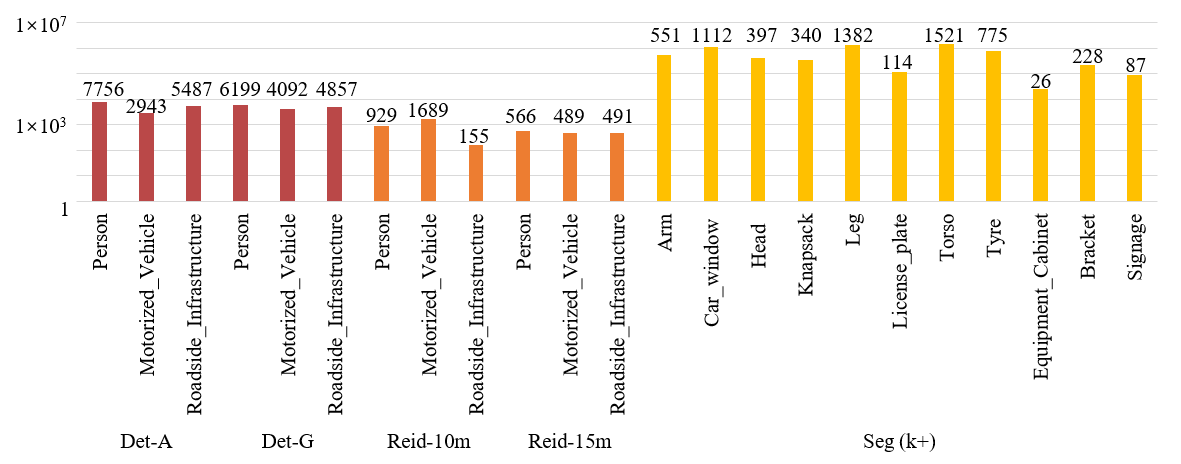}\label{fig:data_stat_b}}\\[0.5em]
  \subfloat[]{\includegraphics[width=0.9\textwidth]{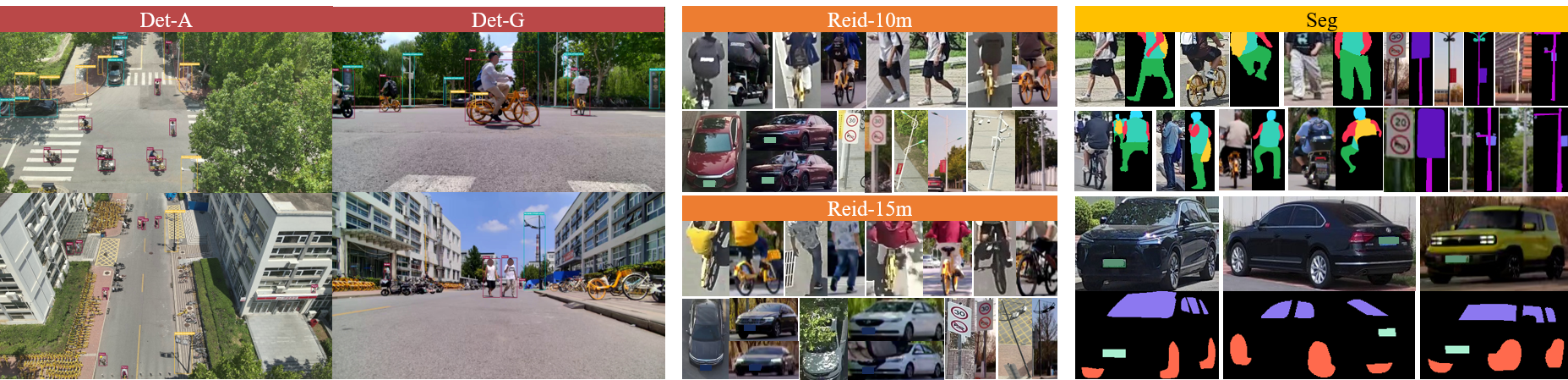}\label{fig:data_stat_c}}
  \caption{Statistical analysis and qualitative examples of AGPC dataset. 
  (a) The distribution of object categories and dataset types. 
  (b) The instance counts for detection and ReID, alongside pixel counts per category for segmentation.
  (c) Visualization of sample frames and annotations for the three core tasks.}
  \label{fig:data_statistics}
\end{figure*}

\section{Related Work}
\label{sec:related_work}

The proposed SCP framework targets progressive air-ground collaborative perception across heterogeneous views and multiple tasks. Therefore, this section briefly reviews related work on air-ground cooperative perception and multi-task learning.

\subsection{Air-Ground Cooperative Perception} 
Air-ground collaborative perception framework, as a paradigm for integrating heterogeneous aerial and ground visual streams, jointly models cross-view information to overcome the perception limits of single-platform sensing. Based on collaborative data dimensionality, existing air-ground methods can be grouped into two categories:
(1) Single-task 2D collaboration: This line treats one platform as a conditional prior for the other and emphasizes task-specific cross-view coupling. For detection, Yu et al.~\cite{R2} formulate air-ground cooperation as a hybrid-genetic task assignment process to optimize role allocation, whereas Chen et al.~\cite{R3} estimate target location from monocular Unmanned Aerial Vehicle (UAV) imagery under severe scale uncertainty. For identity-centric perception, Nguyen et al.~\cite{R4} establish cross-view air-ground ReID protocols with appearance-gap bridging, Nguyen et al.~\cite{R5} extend evaluation to real-world video ReID with stronger temporal and occlusion challenges, and Hambarde et al.~\cite{R6} provide stress-test conditions for UAV-based person detection and re-identification. For temporal association, Yan et al.~\cite{R24} combine aerial global search with ground local tracking, Cui et al.~\cite{R25} introduce hierarchical cooperative multi-target tracking using aerial fisheye cues, and Chen et al.~\cite{R26} perform active localization by coupling visual-inertial cues with single-range constraints.
(2) Single-task 3D collaboration: This branch mainly studies cross-agent 3D perception through staged fusion and cooperative message passing. Hu et al.~\cite{R8} and Gao et al.~\cite{R9} summarize the open challenges of connected perception, including synchronization, communication budget, and calibration dependence. Hou et al.~\cite{R10} provide real-world air-ground driving data, while Zhang et al.~\cite{R11} improve cooperative 3D detection via cross-agent feature fusion. Hao et al.~\cite{R12} further reveal practical bottlenecks in end-to-end V2X pipelines. To improve transferability and robustness, Li et al.~\cite{R13} learn domain-invariant representations for vehicle-infrastructure collaboration.

Existing 3D collaborative perception methods are often impractical in real-world settings due to their reliance on precise geometric alignment. While 2D collaboration offers greater flexibility, current frameworks largely remain confined to single-task cooperation. In contrast, SCP adopts a 2D interaction paradigm and jointly models aerial detection, cross-view search, and semantic segmentation, enabling unified multi-task collaboration under complex operational conditions.

\subsection{Multi-Task Learning}
Multi-Task Learning (MTL) improves generalization by jointly optimizing related tasks while balancing shared priors and task-specific specialization. Following the taxonomy summarized by Vandenhende et al.~\cite{R15}, existing MTL methods can be roughly grouped into three categories:
(1) Shared-representation dense prediction: This line focuses on how multiple tasks interact over common visual features. Xu et al.~\cite{R16} introduce a multi-query transformer to model task-aware interaction for dense prediction. Gao et al.~\cite{TIPPIA21} and Zhang et al.~\cite{TIPMaskSSD20} study unified feature representations for panoptic or instance-level segmentation and detection-style objectives, showing that carefully designed shared embeddings can support multiple dense prediction tasks. Wang et al.~\cite{R18} further extend this direction by adapting Segment Anything as a foundation prior for multi-task fine-tuning.
(2) Scalable and heterogeneous MTL frameworks: This line emphasizes scaling MTL to broader model families, data sources, and modality settings. Chen et al.~\cite{R19} align large vision foundation models for generic visual-linguistic tasks, Lu et al.~\cite{R20} model client heterogeneity in federated multi-task learning, and Sun et al.~\cite{R21} unify multimodal tasks through LLM-based neural tuning.
(3) Flexible task composition and coupled supervision: This line relaxes fixed sharing hierarchies and exploits task complementarity more explicitly. Long et al.~\cite{R22} model multilinear relationships among tasks. For coupled visual supervision, Yang et al.~\cite{TIPMTML25} show that mutual learning between detection and segmentation can significantly enhance infrared small target extraction. Kim et al.~\cite{TIPVPS22} and Wang et al.~\cite{TIPIMT23} further extend such joint optimization to the temporal domain by combining instance-level localization, semantic parsing, and temporal association in video panoptic segmentation.

Most multi-task learning methods are developed under single-view assumptions. When extended to air-ground collaboration, discrepancies in geometry, scale, and occlusion induce distribution shifts that amplify negative transfer under shared-representation paradigms. SCP addresses this limitation by explicitly structuring cross-view and cross-task interactions, enabling transferable and interference-aware representation learning beyond conventional shared representations.

\section{AGPC Benchmark}

This section introduces the AGPC benchmark in a systematic manner. We first define the target task and evaluation metrics, and then describe the key characteristics, statistical properties, and comparative positioning of AGPC with respect to existing benchmarks. Finally, we present the benchmark construction pipeline to show how AGPC is designed to support rigorous evaluation of progressive air-ground collaborative perception.

\subsection{Task and Metrics}
\subsubsection{Task Definition} Unlike traditional independent perception tasks, we designate the process where a designated target observed in the aerial view is associated and localized within the ground view, followed by fine-grained feature analysis, as the air-ground target-guided parsing task. Specifically, this task encompasses four sub-tasks: aerial detection, re-identification, ground detection, and semantic segmentation. Formally, our collaborative framework consists of spatial perception models parameterized by $\Theta_{co} = \{\Theta_{det}, \Theta_{seg}\}$, and a shared dual-layer router parameterized by $\Phi$. We restrict the joint optimization to the spatial sub-task set $\mathcal{T}_{co}=\{det^{a}, det^{g}, seg\}$. To achieve performance coevolution, we optimize both the task representations and the routing mechanism by minimizing the cumulative perception error across this subset:
\begin{equation}
    \mathcal{J}(\Theta_{co}, \Phi) = \sum_{t \in \mathcal{T}_{co}} \mathcal{L}_t (\hat{y}_t, y_t),
\end{equation}
where $\mathcal{L}_t$ represents the specific loss function for sub-task $t$, and the prediction $\hat{y}_t$ is the result optimized by $\Phi$. Meanwhile, the ReID model is optimized independently to preserve identity-level purity; it is used specifically for cross-view association and does not participate in the DLR-driven feature interaction.

\vspace{1ex}
\subsubsection{Evaluation Metrics} 
To evaluate both the holistic system capability and cooperative efficacy, we employ a dual-metric strategy: Average Downstream Performance (ADP) ensures a balanced assessment of multi-granularity perception across heterogeneous tasks, while Coevolutionary Gain ($G_{co}$) quantifies the relative improvements driven by interactive coevolution.
To provide a holistic and balanced evaluation of the cooperative system across heterogeneous tasks, we propose the ADP metric. It is calculated as the geometric mean of three-task downstream performance scores:
\begin{equation}
\resizebox{0.9\linewidth}{!}{$ 
\text{ADP} = \sqrt[3]{ \left(\frac{1}{2} \sum_{d \in \{A, G\}} \text{mAP}_{\text{det}\_d}\right) \cdot \text{mIoU}_{\text{seg}} \cdot \text{Rank-1}_{\text{reid}} }
$}.
\end{equation}
ADP is intended as a supplementary system-level metric rather than a replacement for standard task-wise evaluations. We adopt the geometric mean because the proposed framework is progressive, with coarse localization, cross-view association, and fine-grained parsing being sequentially dependent; consequently, weakness in any stage can bottleneck the overall collaborative performance. Compared with the arithmetic mean, the geometric mean penalizes stage-wise imbalance more strongly and thus better reflects balanced downstream capability. 

To quantify the effectiveness of the coevolutionary interaction facilitated by DLR, we define the coevolutionary gain. It measures the relative improvement of the cooperative perception tasks, specifically aerial detection and ground segmentation, compared to the non-interacting baseline:
\begin{equation}
G_{co} = \frac{1}{2} \left( \frac{P_{\text{det}}^{\text{coop}} - P_{\text{det}}^{\text{base}}}{P_{\text{det}}^{\text{base}}} + \frac{P_{\text{seg}}^{\text{coop}} - P_{\text{seg}}^{\text{base}}}{P_{\text{seg}}^{\text{base}}} \right),
\end{equation}
where $P^{\text{coop}}$ and $P^{\text{base}}$ denote the performance of the proposed SCP and the non-collaborative baseline, respectively.

\subsection{Benchmark Characteristics}
\label{sec:benchmark_features}

Our benchmark distinguishes itself through five key architectural features designed to advance collaborative perception research:

\subsubsection{Spatiotemporally Aligned Video-Based Construction}
Built from 89 synchronized video sequences captured by UAVs and Unmanned Ground Vehicles (UGVs), our benchmark ensures spatiotemporal alignment across views. Its video-based design preserves realistic disturbances such as motion blur, providing a practical testbed for multi-task learning while challenging model robustness in dynamic environments.

\subsubsection{Hierarchical Multi-Granularity Perception}
The three core tasks are derived from the same video streams and integrated into a unified closed-loop framework spanning coarse-grained detection, identity-level ReID, and fine-grained segmentation. This design enables holistic evaluation of localization, identification, and parsing within a shared context, while introducing the challenge of cross-granularity feature alignment.

\subsubsection{Cross-Altitude Domain Adaptability}
To reflect realistic UAV deployment, we collect UAV videos at two altitudes (10m and 15m) and construct the corresponding ReID dataset with an explicit cross-altitude domain gap. This setting encourages scale-invariant representation learning, but also introduces substantial geometric variation and visual degradation, requiring stronger cross-domain generalization.

\subsubsection{High Information Density and Non-Redundancy}
Detection and segmentation data are sampled from the same aligned air-ground videos, with segmentation targets anchored to objects associated through ReID. All object instances are unique, and each image is selected to capture distinctive poses, occlusions, or motion blur, reducing redundancy and discouraging shortcut learning from background cues. This results in a more rigorous evaluation while demanding stronger feature discrimination.

\subsubsection{Cross-Seasonal Diversity}
Furthermore, data acquisition spanned Summer, Autumn, and Winter, providing broad cross-seasonal diversity and featuring challenging environmental variations to verify system robustness.

\begin{table*}[t]
    \caption{Statistical overview of the AGPC dataset. The table details the distribution of samples and pixels across four tasks.  For detection and ReID,
  ``Samples'' denote bounding-box instances (GT boxes for ReID). For segmentation, ``Samples'' denote annotated image instances, while category statistics are additionally reported with pixel-level counts (in Millions `` M'' or Thousands ``k'') to highlight annotation density.}
  \label{tab:dataset_statistics_detailed}
  \begin{center}
    \begin{small}
        \renewcommand{\arraystretch}{1.2}
        \resizebox{0.95\textwidth}{!}{
        \begin{tabular}{l p{10cm} l r r}
          \toprule
          Task & Category Distribution (Instances/Pixels) & Split & Samples & Pct. \\
          \midrule
          
          \multirow{4}{*}{Aerial detection} & 
          \multirow{4}{*}{\parbox{10cm}{\raggedright 
          Instances per Category:\\
          Person (7,756), Motorized Vehicle (2,943), Roadside Infrastructure (5,487)}} & 
          Training & 11,521 & 70\% \\
           & & Validation & 3,176 & 20\% \\
           & & Testing & 1,489 & 10\% \\
           & & Total & 16,186 & 100\% \\
          \cmidrule(r){1-5}
          
          \multirow{4}{*}{Ground detection} & 
          \multirow{4}{*}{\parbox{10cm}{\raggedright 
          Instances per Category:\\
          Person (6,199), Motorized Vehicle (4,092), Roadside Infrastructure (4,857)}} & 
          Training & 10,599 & 70\% \\
           & & Validation & 3,015 & 20\% \\
           & & Testing & 1,534 & 10\% \\
           & & Total & 15,148 & 100\% \\
          \cmidrule(r){1-5}
          
          \multirow{3}{*}{Re-identification} & 
          \multirow{3}{*}{\parbox{10cm}{\raggedright 
          10m Altitude: Person (929), Veh (1,689), Infra (155)\\
          15m Altitude: Person (566), Veh (489), Infra (491)}} & 
          Training & 3,088 & 70\% \\
           & & Testing & 1,231 & 30\% \\
           & & Total & 4,319 & 100\% \\
          \cmidrule(r){1-5}
          
          Ground segmentation &
          \multirow{4}{*}{\parbox{10cm}{\raggedright 
          Pixel Count per Category:\\
          Arm (0.5M), Car window (1.1M), Head (0.4M), Knapsack (0.3M), Leg (1.3M), License plate (0.1M), Torso (1.5M), Tyre (0.7M), Equip. Cabinet (26k), Bracket (0.2M), Signage (87k)
          }} & 
          Training & 3,019 & 70\% \\
           & & Validation & 838 & 20\% \\
           & & Testing & 411 & 10\% \\
           & & Total & 4,268 & 100\% \\
          \bottomrule
        \end{tabular}
        }
    \end{small}
  \end{center}
  \vskip -0.1in
\end{table*}

\subsection{Benchmark Statistics}
\label{sec:benchmark_stats}
The AGPC benchmark is a real-world dataset comprising over 745K raw video frames across 89 scene clips, specifically designed to advance cross-view and cross-task perception. As detailed in Table~\ref{tab:dataset_statistics_detailed}, it provides high-quality annotations for four core tasks, namely aerial detection, ground detection, ReID, and semantic segmentation, all structured through a hierarchical taxonomy. This framework targets three macro-categories for detection and ReID, which are further refined into 11 fine-grained sub-categories for ground segmentation. Quantitatively, the dataset includes 31,334 detection bounding boxes, 4,319 ReID samples, and 4,268 segmentation instances. To ensure robust evaluation, the data is partitioned: detection and segmentation tasks follow a 7:2:1 split, while the ReID task adopts a 7:3 train-test ratio. This distribution enables a coarse-to-fine evaluation of multi-view cooperative perception capabilities.

\subsection{Positioning Against Existing Benchmarks}
\label{sec:benchmark_positioning}
Table~\ref{tab:benchmark_positioning} compares our dataset with representative air-ground collections, cooperative perception benchmarks and other mainstream public datasets. The comparison follows three key dimensions of this work, including targeted collaboration modes, explicit modeling of cross-view interaction, and enhanced cross-task support beyond basic multi-task annotations.

\begin{table*}[t]
    \caption{Positioning AGPC against representative air-ground datasets, cooperative perception benchmarks, and related public datasets. ``Mode'' indicates the collaboration relation, including vehicle-vehicle (Veh-Veh), vehicle-infrastructure (Veh-Inf), air-air (Air-Air), and vehicle-air (Veh-Air).}
    \label{tab:benchmark_positioning}
    \begin{center}
    \begin{small}
        \renewcommand{\arraystretch}{1.2}
        \resizebox{0.95\textwidth}{!}{
        \begin{tabular}{l c c c c c c c c}
            \toprule
            Dataset & Mode & Setting & Cross-view & Air-ground & Cross-task & Detection & ReID & Seg\\
            \midrule
            BDD100K (CVPR 2020)~\cite{B1} & $-$ & 2D & $\xmark$ & $\xmark$ & $\xmark$ & $\cmark$ & $\xmark$ & $\cmark$ \\
            DAIR-V2X (CVPR 2022)~\cite{I3} & Veh-Inf & 3D & $\cmark$ & $\xmark$ & $\xmark$ & $\cmark$ & $\xmark$ & $\xmark$ \\
            MURI (TIP 2023)~\cite{T2} & Inf-Inf & 2D & $\cmark$ & $\xmark$ & $\xmark$ & $\xmark$ & $\cmark$ & $\xmark$ \\
            V2X-Seq (CVPR 2023)~\cite{I4} & Veh-Inf & 3D & $\cmark$ & $\xmark$ & $\xmark$ & $\cmark$ & $\xmark$ & $\xmark$ \\
            Cityscapes-VPS/VIPER (TIP 2022)~\cite{TIPVPS22} & $-$ & 2D & $\xmark$ & $\xmark$ & $\cmark$ & $\cmark$ & $\xmark$ & $\cmark$ \\
            AG-ReID.v2 (TIFS 2024)~\cite{R4} & Inf-Air & 2D & $\cmark$ & $\cmark$ & $\xmark$ & $\xmark$ & $\cmark$ & $\xmark$ \\
            AG-VPReID (CVPR 2025)~\cite{R5} & Inf-Air & 2D & $\cmark$ & $\cmark$ & $\xmark$ & $\xmark$ & $\cmark$ & $\xmark$ \\
            AGC-Drive (NeurIPS 2025)~\cite{R10} & Veh-Air & 3D & $\cmark$ & $\cmark$ & $\xmark$ & $\cmark$ & $\xmark$ & $\xmark$ \\
            Griffin (AAAI 2026)~\cite{I1griffin} & Veh-Air & 3D & $\cmark$ & $\cmark$ & $\xmark$ & $\cmark$ & $\xmark$ & $\xmark$ \\
            \rowcolor[gray]{0.93}
            \textbf{AGPC (Ours)} & \textbf{Veh-Air} & \textbf{2D} & $\cmark$ & $\cmark$ & $\cmark$ & $\cmark$ & $\cmark$ & $\cmark$ \\
            \bottomrule
        \end{tabular}
        }
    \end{small}
    \end{center}
    \vskip -0.1in
\end{table*}

Table~\ref{tab:benchmark_positioning} shows that existing datasets only cover partial aspects of the target problem addressed by AGPC. AG-ReID.v2~\cite{R4} and AG-VPReID~\cite{R5} focus on 2D air-ground collaboration, yet their evaluations are limited to ReID and human-centric perception, rather than progressive cross-task cooperation. Griffin~\cite{I1griffin}, AGC-Drive~\cite{R10}, DAIR-V2X~\cite{I3}, and V2X-Seq~\cite{I4} concentrate on 3D cooperative perception under vehicle-UAV and vehicle-infrastructure modes, with an emphasis on geometric fusion and collaborative detection. Meanwhile, datasets including BDD100K~\cite{B1} and the VCNet unified dataset~\cite{T2} only support partial task coverage and are not tailored for air-ground collaboration. Cityscapes-VPS/VIPER~\cite{TIPVPS22} targets single-view panoptic video parsing without air-ground cross-view coupling. This clearly reveals a notable gap in existing benchmarks, as there is still no 2D vehicle-UAV benchmark that simultaneously supports detection, ReID, and segmentation, and explicitly evaluates progressive cross-task interaction. AGPC is designed to address this gap.

\subsection{Benchmark Construction}
\label{sec:benchmark_gen}

\subsubsection{Data Collection}
To construct the benchmark, we developed a synchronized UAV-UGV platform for real-world cross-view data collection. The UGV utilized a 3-axis stabilized 4K gimbal camera (SIYI ZT6), while the UAV (DJI Mavic 3TA) employed its wide-angle and telephoto sensors to capture high-resolution imagery. To ensure consistent cross-view overlap, the platforms were spatially coordinated: the UAV maintained a $45^{\circ}$ oblique view at altitudes of 10m and 15m, synchronized with the horizontal front-facing perspective of the UGV.

\subsubsection{Annotation Strategy} 
We provide high-quality 2D annotations for three core tasks: object detection, ReID, and semantic segmentation. All annotations were manually generated to ensure ground-truth fidelity. 
All benchmark annotations are derived from the same raw video corpus. Specifically, detection images are obtained by performing quality-controlled and redundancy-aware key frame sampling from the raw video frames in both aerial and ground streams. The ReID dataset is further constructed from the same raw video corpus by selecting synchronized cross-view target samples and assigning identity labels to their ground-truth bounding boxes. Based on the established ReID image set, the segmentation dataset is further annotated on selected ReID images, making the segmentation task a fine-grained subset of the ReID benchmark. This design ensures that detection, ReID, and segmentation are progressively connected rather than independently collected.


\section{Socialized Co-Perception}

This section first introduces the two fundamental principles underlying the SCP air-ground collaborative framework, namely sociability and complementarity, and then elaborates on the overall architecture of SCP and the working principle of its core module, DLR.

\subsection{Sociability and Complementarity}
This section explores two fundamental principles of air-ground collaboration from an information-theoretic perspective. First, positive social interactions between cross-view air-ground models are beneficial, demonstrating the importance of sociability. Second, inherent cross-view visual disparities and task constraints can lead to cross-view interference, thereby reducing the effectiveness of auxiliary information. Therefore, effectively filtering complementary information between models is also crucial for air-ground collaboration. These theoretical insights motivate our design, specifically the cross-task coevolution in SCP and the multi-layer routing filtration in DLR.

\begin{figure*}[t]
  \vskip 0.2in
  \begin{center}
    \centerline{\includegraphics[width=\textwidth]{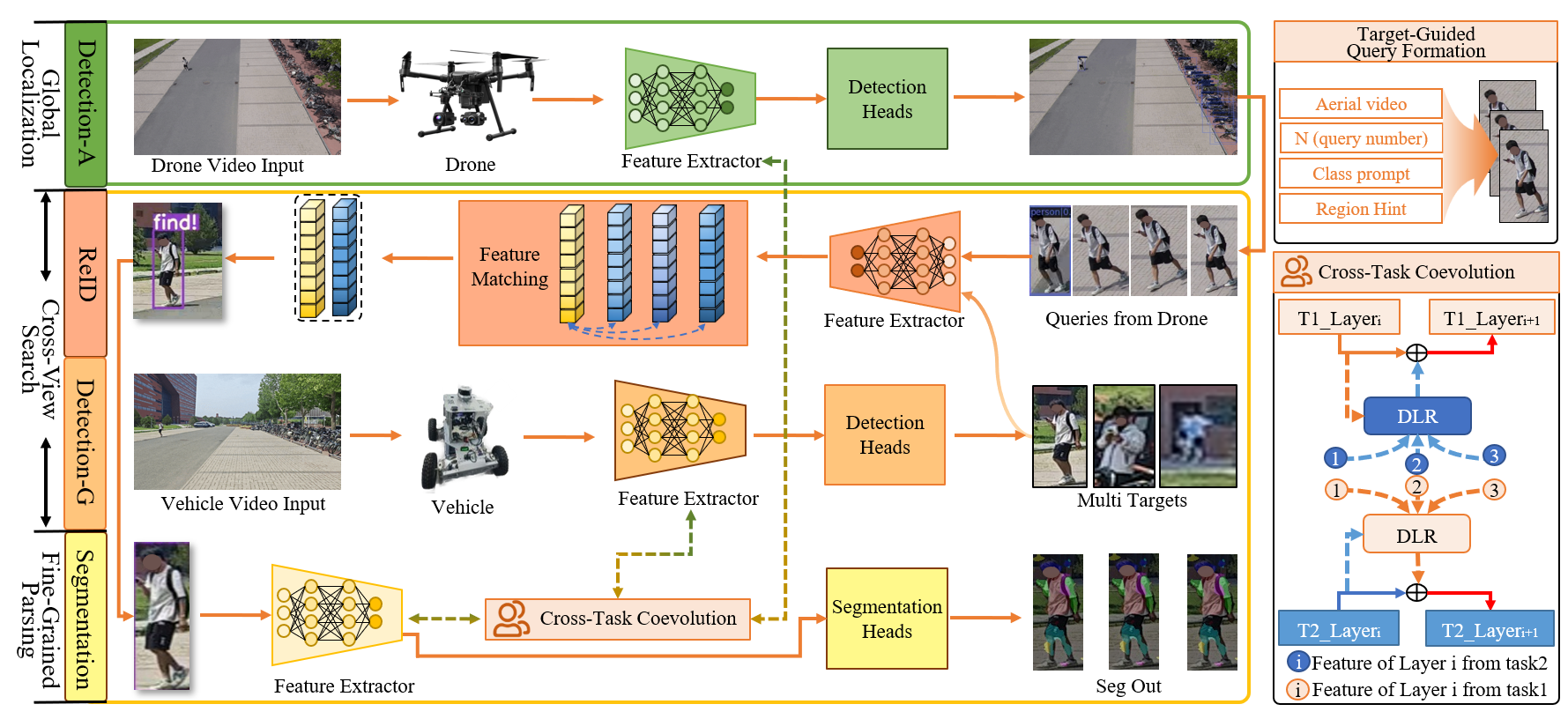}}
    \caption{The proposed SCP framework. It establishes a progressive inference pipeline comprising three stages: (1) Global localization; (2) Cross-view search; and (3) Semantic segmentation. The right panel illustrates the cross-task coevolution, utilizing the DLR and main-auxiliary training strategy.}
    \label{fig:SCP}
  \end{center}
\end{figure*}

\label{sec:sociability}

\noindent \textbf{Problem setup:}
Let $\mathcal{S} = \{\mathcal{M}_{m}, \mathcal{M}_{a}\}$ be a society with a main task model $\mathcal{M}_{m}$ and an auxiliary model $\mathcal{M}_{a}$. Let $X_{m}, X_{a} \in \mathcal{X}$ denote their respective input space, and $Y \in \mathcal{Y}$ the target label space for the main task.
In the isolated setting, the prediction is based on $P(Y|X_{m})$. In the coevolutionary setting, the system utilizes the joint information $P(Y|X_{m}, X_{a})$.

\begin{definition}[Sociability Information]
    \label{def:sociability}
    The sociability information $I_{s}$ provided by the auxiliary view is defined as the conditional mutual information between $X_{a}$ and $Y$, given $X_{m}$:
    \begin{equation}
        \resizebox{0.8\linewidth}{!}{$
            I_{s} = I(X_{a}; Y | X_{m}) = H(Y | X_{m}) - H(Y | X_{m}, X_{a}).
        $}
    \end{equation}
    Here, $I_{s} > 0$ implies that $\mathcal{M}_{a}$ contributes unique information reducing the uncertainty of $Y$.
\end{definition}

\noindent \textbf{Bayes error rate:} The Bayes error rate~\cite{Bayes_error_rate} represents the minimum achievable probability of error. We define the error bounds for the isolated main task ($P^{iso}_{e}$) and the coevolutionary society ($P^{co}_{e}$) as follows:
\begin{equation}
    P^{iso}_{e} \leq 1 - \exp\big(-H(Y|X_{m})\big),
    \label{eq:bayes_bound_iso}
\end{equation}
\begin{equation}
    P^{co}_{e} \leq 1 - \exp\big(-H(Y|X_{m}, X_{a})\big).
    \label{eq:bayes_bound_co}
\end{equation}
 
\begin{proposition}[Necessity of Sociability]
    \label{thm:error_reduction}
    Building on prior work~\cite{icml2024,theory2,theory3}, if the sociability information $I_{s} > 0$, the upper bound of the Bayes error rate for the collaborative society is lower than that of the isolated view, $\text{Upper}(P^{co}_{e}) < \text{Upper}(P^{iso}_{e})$.
\end{proposition}

\begin{proof}
    Based on Eq.~\eqref{eq:bayes_bound_co}, the error bound for the society is expressed as follows:
    \begin{equation}
        \resizebox{0.6\linewidth}{!}{$
            P^{co}_{e} \leq 1 - \exp\big(-H(Y|X_{m}, X_{a})\big).
        $}
    \end{equation}
    Substituting the entropy reduction from Def.~\ref{def:sociability}, we can derive:
    \begin{equation}
        \resizebox{0.6\linewidth}{!}{$
            \begin{aligned}
                P^{co}_{e} &\leq 1 - \exp\big(-(H(Y|X_{m}) - I_{s})\big).
            \end{aligned}
        $}
    \end{equation}
    Consequently
    \begin{equation}
        \resizebox{0.6\linewidth}{!}{$
            \begin{aligned}
                P^{co}_{e} &\leq 1 - \exp\big(-H(Y|X_{m})\big) \cdot \exp(I_{s}).
            \end{aligned}
        $}
    \end{equation}
    Since $I_{s} > 0$, we have $\exp(I_{s}) > 1$. Comparing this with the isolated bound, we have:
    \begin{equation}
        \resizebox{0.8\linewidth}{!}{$
            1 - \exp\big(-H(Y|X_{m})\big) \cdot \exp(I_{s}) < 1 - \exp\big(-H(Y|X_{m})\big).
        $}
    \end{equation}
    Thus, we conclude:
    \begin{equation}
        \resizebox{0.4\linewidth}{!}{
           $\text{Upper}(P^{co}_{e}) < \text{Upper}(P^{iso}_{e}).$
        }
    \end{equation}
\end{proof}

\begin{definition}[Task-Relevant Information Degradation under Interference]
    \label{def:interference_degradation}
    Let $X_{a}^{*}$ denote the ideal, interference-free auxiliary representation, and $X_{a}$ denote the interference-corrupted auxiliary representation. Due to interference, the task-relevant information of the corrupted input ($I_{s}^{base}$) is less than that of the ideal input ($I_{s}^{ideal}$):
    \begin{equation}
        I_{s}^{base} = I(X_{a}; Y | X_{m}) < I(X_{a}^{*}; Y | X_{m}) = I_{s}^{ideal}.
    \end{equation}
    We define $\Delta = I_{s}^{ideal} - I_{s}^{base} > 0$ as the information gap caused by interference.
\end{definition}

\begin{proposition}[Necessity of Complementarity Filtering]
    \label{thm:complementarity_necessity}
    The Bayes error upper bound of the system using base auxiliary data ($P^{base}_{e}$) is higher than that of the system using ideal filtered data ($P^{ideal}_{e}$). Thus, a filtering mechanism is theoretically motivated under this formulation to minimize the error bound.
\end{proposition}

\begin{proof}
    Let $P^{base}_{e}$ and $P^{ideal}_{e}$ be the error bounds for the base and ideal scenarios. Based on Proposition~\ref{thm:error_reduction}, they are expressed as follows:
\begin{equation}
    P^{base}_{e} \leq 1 - \exp\big(-H(Y|X_{m})\big) \cdot \exp(I_{s}^{base}),
    \label{eq:bayes_error_base}
\end{equation}
\begin{equation}
    P^{ideal}_{e} \leq 1 - \exp\big(-H(Y|X_{m})\big) \cdot \exp(I_{s}^{ideal}).
    \label{eq:bayes_error_ideal}
\end{equation}
    Substituting $I_{s}^{ideal} = I_{s}^{base} + \Delta$ from Def.~\ref{def:interference_degradation}, we can derive:
    \begin{equation}
        \resizebox{0.8\linewidth}{!}{$
            P^{ideal}_{e} \leq 1 - \exp\big(-(H(Y|X_{m}) - (I_{s}^{base} + \Delta))\big).
        $}
    \end{equation}
    Consequently,
    \begin{equation}
        \resizebox{0.8\linewidth}{!}{$
            P^{ideal}_{e} \leq 1 - \underbrace{\exp\big(-H(Y|X_{m})\big) \cdot \exp(I_{s}^{base})}_{\text{Term}_{base}} \cdot \exp(\Delta).
        $}
    \end{equation}
    Since $\Delta > 0$, we have:
    \begin{equation}
        \resizebox{0.5\linewidth}{!}{$
            \text{Term}_{base} \cdot \exp(\Delta) > \text{Term}_{base}.
        $}
    \end{equation}
    Thus, we conclude:
    \begin{equation}
        \resizebox{0.25\linewidth}{!}{$
            P^{ideal}_{e} < P^{base}_{e}.
        $}
    \end{equation}
    This indicates that recovering $X_{a}^{*}$ is necessary to achieve the optimal error bound.
\end{proof}

\begin{figure*}[t]
  \vskip 0.2in
  \begin{center}
    \centerline{\includegraphics[width=0.8\textwidth]{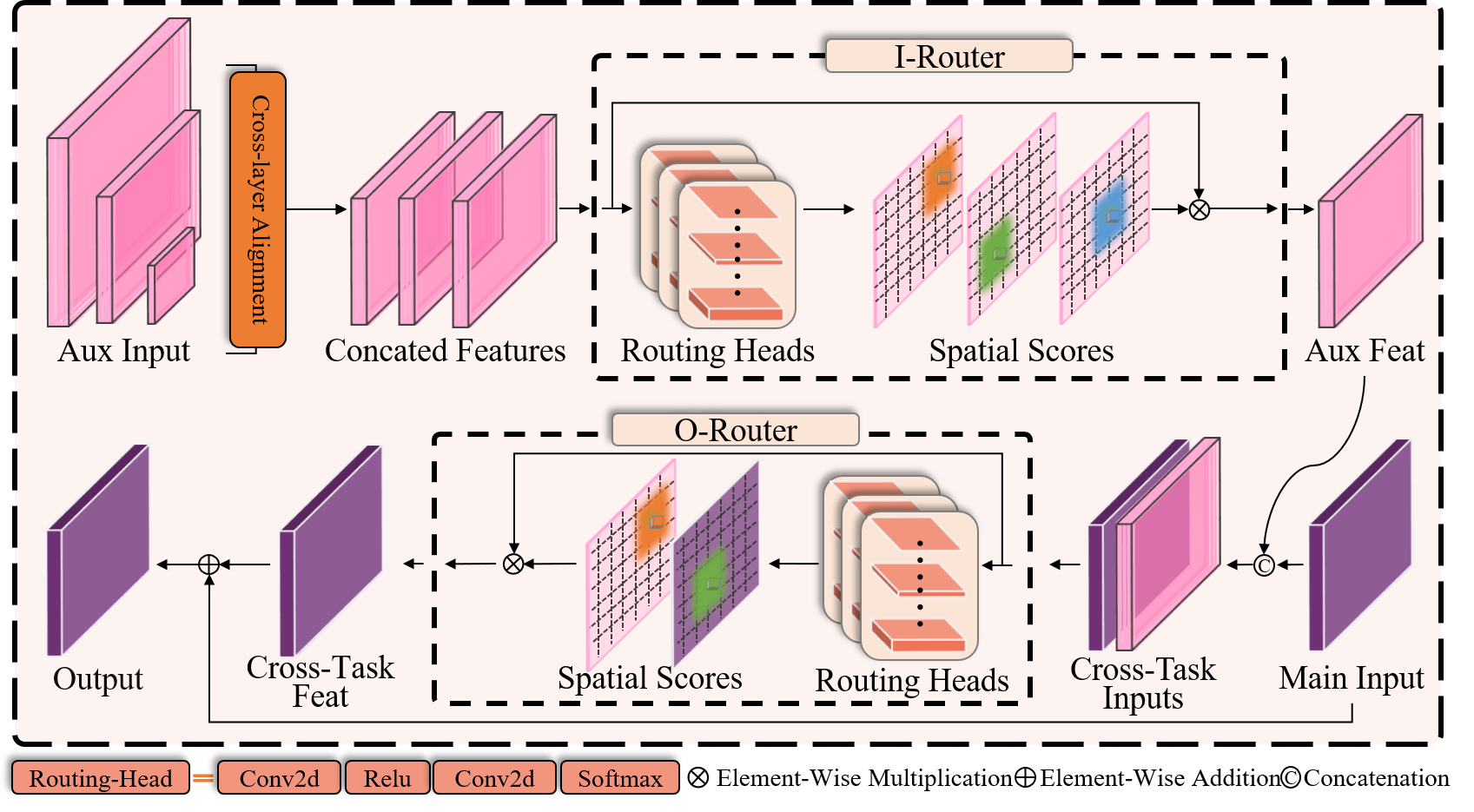}}
    \caption{Illustration of the DLR. It comprises the I-Router for adaptive multi-scale feature aggregation and the O-Router for dynamic residual gating, facilitating robust cross-task knowledge transfer.}
    \label{fig:DLR}
  \end{center}
\end{figure*}

\subsection{Methodology}

\subsubsection{Socialized Co-Perception Framework}
\label{sec:framework}
To overcome task isolation and achieve efficient target locking and segmentation, we propose the SCP framework. As illustrated in Figure~\ref{fig:SCP} and Algorithm~\ref{alg:inference}, SCP operates as a unified ecosystem. Core to this framework is the air-ground cross-task coevolution mechanism, where the DLR module $\Phi$ facilitates the joint optimization of detection and segmentation. This process yields the evolved parameters $\Theta^*_{det}$ and $\Theta^*_{seg}$, which encapsulate shared multi-granularity knowledge. Leveraging these evolved capabilities, SCP executes a progressive air-ground inference pipeline:
\begin{equation}
\scalebox{0.9}{$
M_{refine} = \mathcal{M}_{seg}^g \left(
\Psi_{match} \left(
\underbrace{\mathcal{E}(\mathcal{M}_{det}^a)}_{q_{sem}},
\mathcal{M}_{det}^g
\right)
\right)
$}.
\end{equation}
where $\mathcal{M}_{det}^a$, $\mathcal{M}_{det}^g$, and $\mathcal{M}_{seg}^g$ denote the aerial detection, ground detection, and ground segmentation networks, respectively, parameterized by the co-evolved weights $\Theta^*_{det}$ and $\Theta^*_{seg}$. The operator $\mathcal{E}(\cdot)$ transforms the aerial detection results into a high-level semantic query vector $q_{sem}$, while $\Psi_{match}(\cdot)$ performs cross-view association to lock the target among ground candidates. Finally, $M_{refine}$ represents the output identity-aware fine-grained segmentation mask.

\begin{algorithm}[H]
  \caption{SCP Framework.}
  \label{alg:inference}
  \textbf{Input}: DLR module $\Phi$, aerial video stream $I_a$, ground video stream $I_g$; \\
  \textbf{Output}: Fine-grained mask sequence $\mathcal{V}_{out}$;
  
  \begin{algorithmic}[1]
    \STATE \COMMENT{\textit{Step 1: Cross-task coevolution}}
    \STATE $\Theta^*_{det}, \Theta^*_{seg}, \Phi^* \leftarrow \text{Coevolve}(\Theta_{det}, \Theta_{seg}; \Phi)$;
    \STATE \COMMENT{\textit{Step 2: Execute progressive inference pipeline}}
    
    \STATE {\bfseries Phase 1: Global Localization}
    \STATE $\mathcal{B}_a \leftarrow \mathcal{M}_{det}^a(I_a; \Theta^*_{det}, \Theta^*_{seg}, \Phi^*)$;
    
    \STATE {\bfseries Phase 2: Cross-View Search}
    \STATE $q_{sem} \leftarrow \mathcal{E}(\mathcal{B}_a)$;
    
    \FOR{frame $I_g^t$ in $I_g$}
      \STATE $\mathcal{C}_g \leftarrow \mathcal{M}_{det}^g(I_g^t; \Theta^*_{det}, \Theta^*_{seg}, \Phi^*)$;
      \STATE $\hat{o}_g^t \leftarrow \Psi_{match}(q_{sem}, \mathcal{C}_g)$;
      \IF{$\hat{o}_g^t$ is valid}
        \STATE {\bfseries Phase 3: Identity-Aware Fine-Grained Parsing}
        \STATE $M_{refine} \leftarrow \mathcal{M}_{seg}^g(\hat{o}_g^t; \Theta^*_{seg}, \Theta^*_{det}, \Phi^*)$;
        \STATE $\mathcal{V}_{out}.\text{append}(M_{refine})$;
      \ENDIF
    \ENDFOR
    
    \STATE {\bfseries Return} $\mathcal{V}_{out}$;
  \end{algorithmic}
\end{algorithm}

\noindent \textbf{Phase 1: Global Localization.}
Operating as the global observer, the UAV inputs aerial video streams ($I_{a}$) into the aerial detector initialized with the co-evolved parameters $\{\Theta^*_{det}, \Theta^*_{seg}\}$ and the optimized DLR module $\Phi^*$. By dynamically routing auxiliary features via $\Phi^*$, the detector enhances its representation capabilities, generating robust global candidate bounding boxes $\mathcal{B}_a$:
\begin{equation}
 \mathcal{B}_a \leftarrow \mathcal{M}_{det}^a(I_a; \Theta^*_{det}, \Theta^*_{seg}, \Phi^*).
\label{eq:Global_localization}
\end{equation}

\noindent \textbf{Phase 2: Cross-View Search.}
This module bridges the viewpoint gap to lock the target on the ground. First, the aerial target region is cropped from the detected aerial candidates and transformed by $\mathcal{E}(\cdot)$ into a semantic query embedding. In implementation, $\mathcal{E}(\cdot)$ is instantiated by the ReID encoder rather than an additional learnable query head: each aerial crop is resized to $256\times128$, normalized, and fed into a ResNet-50-IBN-a network followed by global average pooling and a BN neck, yielding a 2048-dimensional feature that is further $\ell_2$-normalized to form the query representation $q_{sem}$. Second, the ground vehicle acts as the local executor. It generates ground candidates $\mathcal{C}_g$ from ground video streams $I_g^t$ using the coupled parameter set and the routing module $\Phi^*$:
\begin{equation}
\mathcal{C}_g = \mathcal{M}_{det}^g(I_g^t; \Theta^*_{det}, \Theta^*_{seg}, \Phi^*).
\label{eq:ground_candidates}
\end{equation}
Thus, the system computes the matching similarity between $q_{sem}$ and $\mathcal{C}_g$ to return the specific target $\hat{o}_{g}^t$:
\begin{equation}
\hat{o}_g^t \leftarrow \Psi_{match}(q_{sem}, \mathcal{C}_g).
\label{eq:matching}
\end{equation}

\noindent \textbf{Phase 3: Identity-Aware Fine-Grained Parsing.}
Once the target is spatially locked, the framework performs pixel-level refinement. The bounding box of the valid target $\hat{o}_g^t$ is used to crop the region of interest. This RoI is then processed by the segmentation network. To ensure precise boundary delineation, the segmentation model, parameterized by $\Theta^*_{seg}$, is conditionally gated by $\Phi^*$ to incorporate auxiliary detection knowledge from $\Theta^*_{det}$. This step captures fine-grained attributes invisible from the aerial view:
\begin{equation}
M_{refine} \leftarrow \mathcal{M}_{seg}^g(\hat{o}_g^t; \Theta^*_{seg}, \Theta^*_{det}, \Phi^*).
\label{eq:seg}
\end{equation}

\subsubsection{DLR-Driven Cross-Task Coevolution}
As formally derived in Proposition~\ref{thm:error_reduction} and Proposition~\ref{thm:complementarity_necessity}, the performance evolution of air-ground models similarly depends on sociability and complementarity. To this end, we employ the main-auxiliary training strategy (Algorithm~\ref{alg:training}) to facilitate social interaction across models, as illustrated in Figure~\ref{fig:SCP}. Furthermore, to ensure the complementarity of interactive information, we propose the DLR module, a hierarchical gating module inserted between different task backbones that reformulates feature interaction as a dense expert routing process.

\begin{algorithm}[H]
  \caption{Main-Auxiliary Training Strategy.}
  \label{alg:training}
  \textbf{Input}: Datasets $\mathcal{D}_{det}, \mathcal{D}_{seg}$, models $\Theta_{det}, \Theta_{seg}$, DLR $\Phi$, main task $m \in \{det, seg\}$; \\
  \textbf{Output}: Optimized $\Theta^*_{det}, \Theta^*_{seg}, \Phi^*$;

  \begin{algorithmic}[1]
    \STATE \textbf{Phase 1: Expert Pretraining}
    \STATE Initialize $\Theta_{det}, \Theta_{seg}$ independently on their respective datasets;

    \STATE \textbf{Phase 2: Main Task Training}
    \STATE Define auxiliary task $a \in \{det, seg\} $;
    \STATE Freeze $\Theta_a$, unfreeze $\Theta_m, \Phi$;
    
    \WHILE{not converged}
      \STATE $F_{m}, F_{a} \leftarrow \Theta_m(\mathcal{D}_m), \Theta_a(\mathcal{D}_m)$;
      \STATE $F_{out} \leftarrow \text{DLR}(F_{m}, F_{a}; \Phi)$;
      \STATE Update $\Theta_m, \Phi$ by minimizing $\mathcal{L}_m(F_{out}, Y_m)$;
    \ENDWHILE
    
    \STATE {\bfseries Return} $\Theta^*_{det}, \Theta^*_{seg}, \Phi^*$;
  \end{algorithmic}
\end{algorithm}

As illustrated in Figure~\ref{fig:DLR}, at the $l$-th interaction layer, let the main-branch feature be $F_{main} \in \mathbb{R}^{C_l \times H_l \times W_l}$, where $C_l \in \{256,512,1024\}$ corresponds to layer1, layer2, and layer3, respectively. We first define multi-scale auxiliary features $\mathcal{F}_{aux} = \{f_k\}$ as a set of candidate experts, where $k$ represents varying semantic depths. After bilinear interpolation and $1\times1$ Conv-BN-ReLU projection, each candidate expert is aligned to the same spatial and channel dimensions as the main branch, namely $f_k \in \mathbb{R}^{C_l \times H_l \times W_l}$. To prioritize the most informative granularity without hard selection, parallel routing heads $\mathcal{G}_{in}$ generate a spatially adaptive consensus $F'_{aux}$ via dense soft-routing,
\begin{equation}
F'_{aux} = \sum_{k} \left( \mathrm{Softmax}\left(\frac{\mathcal{G}_{in}(f_k)}{\tau}\right) \odot f_k \right).
\label{eq_dlr_in_router}
\end{equation}
where $\mathcal{G}_{in}(f_k) \in \mathbb{R}^{1 \times H_l \times W_l}$ denotes the spatial routing score of the $k$-th expert, and the Softmax is performed across the expert dimension at each spatial location. Therefore, the input of I-Router is $\mathcal{F}_{aux} = \{f_k\}$, and its output is the aggregated auxiliary feature $F'_{aux} \in \mathbb{R}^{C_l \times H_l \times W_l}$. Here, $\tau$ modulates the routing sharpness. Functioning as a pixel-wise scale selector, this mechanism dynamically aggregates the most salient granularity for each spatial coordinate, thereby ensuring scale-invariance before interaction.

Subsequently, to control the information flow and filter negative transfer, we employ a residual gating mechanism. We employ a gating network $\mathcal{G}_{out}$ to predict a pair of modulation coefficients $[\alpha_{main}, \alpha_{aux}]$ based on the joint representation
\begin{equation}
Z = \mathrm{Concat}(F_{main}, F'_{aux}),
\label{eq_dlr_joint_repr}
\end{equation}
where $Z \in \mathbb{R}^{2C_l \times H_l \times W_l}$. Then,
\begin{equation}
[\alpha_{main}, \alpha_{aux}] = \mathrm{Softmax}\left(\frac{\mathcal{G}_{out}(Z)}{\tau}\right),
\label{eq_dlr_out_router}
\end{equation}
where $\alpha_{main}, \alpha_{aux} \in \mathbb{R}^{1 \times H_l \times W_l}$. Therefore, the input of O-Router is the joint representation $\mathrm{Concat}(F_{main}, F'_{aux})$, and its output is a pair of spatial gating maps that task-condition the retention of main-branch features and the injection of auxiliary features, respectively. These gates govern the selective injection process via a residual update,
\begin{equation}
F_{out} = \mathcal{R}(F_{main}, F'_{aux}; \alpha_{main}, \alpha_{aux}),
\label{eq_dlr_residual_update}
\end{equation}
where $\mathcal{R}(\cdot)$ denotes the residual fusion function parameterized by the predicted spatial gates, and $F_{out} \in \mathbb{R}^{C_l \times H_l \times W_l}$. This design effectively decouples the feature space, allowing the model to prioritize intrinsic representations where the main view is sufficient, while selectively absorbing complementary cues to resolve ambiguities.

\section{Experiments}

In this section, we use ResNet-50 as the common backbone. We employ Faster R-CNN, DeepLabv3, and ResNet-50-IBN as baselines for the detection, segmentation, and ReID tasks, respectively. Specifically, we conduct the evaluation covering comparisons with representative adapted interaction methods, ablation studies on DLR modules, and holistic efficacy analysis. We further discuss system robustness across varying collaboration scenarios and UAV altitudes. More experiments and implementation details are provided in the following subsections.

\subsection{Implementation Details}
\label{sec:Implementation_Details}

We evaluate the proposed method on the AGPC dataset across four subsets: aerial detection, ground detection, ground segmentation, and air-ground ReID integrated from 10m and 15m altitudes. 

This section details the comparative methods and datasets utilized in our experiments. All models were implemented using the PyTorch framework and trained on eight NVIDIA RTX 3090 GPUs. Specifically, for detection-centric models, we employed Stochastic Gradient Descent (SGD) as the optimizer with a momentum of 0.9 and a weight decay of $1 \times 10^{-4}$. These models utilized the SmoothL1Loss function and were trained for 12 epochs with a batch size of 2 and an initial learning rate of 0.02. For the ReID task, we adopted the Adam optimizer. The models were trained for 120 epochs with a batch size of 8. The learning rate was initialized at $3.5 \times 10^{-4}$. The objective function incorporated Triplet Loss and Cross-Entropy Loss with label smoothing. Finally, regarding segmentation-centric models, we utilized SGD (momentum 0.9, weight decay $1 \times 10^{-4}$) combined with the CrossEntropyLoss function. These models were trained for 20,000 iterations with a batch size of 4 and a learning rate of 0.01.

\subsection{Compared Methods}

To validate the effectiveness of DLR, we compare it with several representative interaction methods from recent multi-task and heterogeneous learning literature. These baselines cover different interaction paradigms, including general knowledge transfer, selective collaboration, dynamic task routing, and cross-modal alignment, thereby providing a comprehensive reference for evaluating interaction mechanisms in our air-ground cross-view setting.

\begin{itemize}
    \item \textbf{NKT}~\cite{E4Intern} builds a general vision ecosystem through multi-stage pre-training.
    \item \textbf{DISC}~\cite{E1-dsc} enhances multi-task learning via dynamic hierarchical selective collaboration.
    \item \textbf{DTF}~\cite{E3-tdf} improves task-specific representation through adaptive transformer path routing.
    \item \textbf{TVAM}~\cite{E2-tvam} bridges cross-modal gaps via guidance-driven heterogeneous feature alignment.
    \item \textbf{BCSI}~\cite{BCSI} enables selective cross-task transfer via bidirectional channel interaction.
    \item \textbf{ACFA}~\cite{ACFA} enhances task interaction via adaptive cross-fusion attention.
    \item \textbf{DDIM}~\cite{DDIM} captures spatial dependencies through dual-dimensional branch interaction.
    \item \textbf{SSFM}~\cite{SSFM} fuses semantic and structural cues for cross-task integration.
\end{itemize} 
Since these methods were not originally designed for the air-ground cross-task framework, we adapt them under a unified protocol by keeping all modules unchanged except for replacing the DLR interaction block. These modules are inserted at the same feature interaction interface as DLR, while ReID is kept as an independent branch. This ensures a fair comparison of interaction mechanisms under the same air-ground cross-view setting.

\subsubsection{Necessity of Socialized Collaboration} The data in Table~\ref{tab:Parallel_compact_DLR} indicate that all model variants equipped with interaction mechanisms achieve considerably better performance than the Vanilla baseline without social collaboration. By constructing bidirectional interactive correlations between aerial global cues and ground fine-grained cues, the collaborative gain of the model is effectively improved. This sufficiently demonstrates that in scenarios with significant air-ground feature heterogeneity, social correlation between aerial and ground models serves as a key basis for achieving steady performance improvement.

\subsubsection{Complementarity Based on Task-Conditioned Routing}DLR further surpasses all adapted interaction modules, exhibiting its most significant advantage on the segmentation task, where it improves by 2.17\% over the runner-up, and ultimately yielding a 3.73\% increase in the $G_{co}$ metric. This result suggests that the core challenge in air-ground multi-task perception lies in regulating the content, semantic scale, and downstream objectives of cross-view information transfer. By decoupling spatial expert selection from task-conditioned output modulation, DLR effectively extracts fine-grained complementary information from cross-view features, thereby demonstrating a significant advantage.

\begin{table}[H]
    \caption{Comparison of performance across different methods on AGPC. The $1^{st}$/$2^{nd}$ best results are
  indicated in {\color{red}red}/{\color{blue}blue}.}
    \label{tab:Parallel_compact_DLR}
    \centering
    \footnotesize
    \begin{tabular*}{\linewidth}{@{\extracolsep{\fill}}lccc@{}}
      \toprule
      Method & Det-A(\%)$\uparrow$ & Seg(\%)$\uparrow$ & \boldmath$G_{co}$(\%)$\uparrow$ \\
      \midrule
      Vanilla & 76.12 & 37.72 & - \\
      NKT~\cite{E4Intern} & 88.07 & 52.82 & 27.87 \\
      DISC~\cite{E1-dsc} & 87.15 & 44.60 & 17.32 \\
      DTF~\cite{E3-tdf} & 88.68 & 52.17 & 27.40 \\
      TVAM~\cite{E2-tvam} & 88.45 & 46.19 & 19.33 \\
      ACFA~\cite{ACFA} & 88.28 & \textbf{\color{blue}54.03} & \textbf{\color{blue}29.61} \\
      BCSI~\cite{BCSI} & 88.08 & 53.86 & 29.25 \\
      DDIM~\cite{DDIM} & \textbf{\color{blue}88.73} & 53.31 & 28.95 \\
      SSFM~\cite{SSFM} & 88.07 & 53.29 & 28.68 \\
      DLR (Ours) & ${\textbf{\color{red}89.59}_{\color{red}(+0.86)}}$ & ${\textbf{\color{red}56.20}_{\color{red}
  (+2.17)}}$ & ${\textbf{\color{red}33.34}_{\color{red}(+3.73)}}$ \\
      \bottomrule
    \end{tabular*}
    \vskip -0.1in
  \end{table}

\subsection{Ablation Analysis of DLR Modules}
\subsubsection{Hierarchical Mechanism of Dual Routing}
The ablation study in Table~\ref{tab:ablation_Router} reveals the hierarchical interaction of the dual routing modules. Enabling the I-Router or O-Router individually raises $G_{co}$ to 26.05\% and 29.66\%, respectively, demonstrating that input-side expert filtering and output-side task-conditioned modulation can independently optimize feature fusion. When both modules are activated, the full DLR improves upon the no-routing baseline by 0.69\%, 6.79\%, and 9.45\% on Det-A, Seg, and $G_{co}$, respectively. These step-wise gains suggest that effective air-ground interaction follows a coarse-to-fine hierarchical mechanism: the I-Router first suppresses irrelevant multi-scale responses and retains transferable auxiliary cues, after which the O-Router selectively allocates features according to the demands of downstream tasks.

\subsubsection{Asymmetric Gains Across Task Granularities} DLR exhibits asymmetric performance gains across detection and segmentation tasks. The relatively small absolute improvement on Det-A reflects the limited optimization space for the coarse-grained localization task on the existing dataset: AGPC contains only three macro-detection categories, and strong interaction baselines are already saturated on Det-A (87.15\%--88.68\%). In contrast, the segmentation task requires pixel-level discrimination for dense categories, imposing stricter requirements on the semantic compatibility of cross-view cues. The more significant gain in segmentation metrics demonstrates that as the system transitions from coarse-grained target locking to fine-grained scene parsing, the routing mechanism plays a critical role in driving performance gains.

\begin{table}[H]
  \caption{Ablation studies of the effect of related components. $\checkmark$ denotes DLR with the corresponding module enabled.}
  \label{tab:ablation_Router}
  \centering
  \scriptsize
  \setlength{\tabcolsep}{2pt}
  \resizebox{\linewidth}{!}{
    \begin{tabular}{@{}lccccc@{}}
      \toprule
      Method & I-Router & O-Router & Det-A (\%) $\uparrow$ & Seg (\%) $\uparrow$ & \boldmath$G_{co}$ (\%) $\uparrow$ \\
      \midrule
      \multirow{4}{*}{DLR} 
      & - & - & 88.90 & 49.41 & 23.89 \\
      & $\checkmark$ & - & 89.53 & 51.04 & 26.05 \\
      & - & $\checkmark$ & 89.54 & 53.26 & 29.66 \\
      & $\checkmark$ & $\checkmark$ & ${\textbf{\color{red}89.59}_{\color{red}(+0.69)}}$ & ${\textbf{\color{red}56.20}_{\color{red}(+6.79)}}$ & ${\textbf{\color{red}33.34}_{\color{red}(+9.45)}}$ \\
      \bottomrule
    \end{tabular}
  }
  \vskip -0.1in
\end{table}

\subsection{Overall-Level Performance Analysis of SCP}
\subsubsection{System-Level Optimization} As shown in Table~\ref{tab:Compact_SCP}, the introduction of the DLR module improves aerial and ground detection performance by 13.47\% and 10.90\%, segmentation performance by 18.48\%, and the overall ADP metric by 7.86\%. The performance improvements in localization and parsing validate that SCP facilitates effective complementation between coarse-grained global cues and fine-grained downstream knowledge. This enables robust cross-view and cross-task collaboration among air-ground individuals, ultimately realizing effective system-level optimization.

\subsubsection{Feature Decoupling and Collaboration Strategy} The performance of ReID in Table~\ref{tab:Compact_SCP} remains unchanged at 21.20\% because this task does not participate in the DLR-driven joint optimization. This setup aims to maintain the representational independence of the identity matching task while spatially aware tasks undergo collaborative optimization, avoiding negative interference between tasks.
\begin{table}[H]
  \caption{Comparison of SCP performance across all sub-tasks. CC denotes cross-task coevolution.}
  \label{tab:Compact_SCP}
  \centering
  \setlength{\tabcolsep}{1.5pt}
    \resizebox{\linewidth}{!}{
    \begin{tabular}{@{}lccccc@{}}
      \toprule
      Method & Det-A (\%) $\uparrow$ & Det-G (\%) $\uparrow$ & Seg (\%) $\uparrow$ & ReID (\%) $\uparrow$ & ADP (\%) $\uparrow$ \\
      \midrule
      SCP w/o CC    & 76.12 & 71.25 & 37.72 & 21.20 & 38.91 \\
      SCP w/ CC & ${\textbf{\color{red}89.59}_{\color{red}(+13.47)}}$ & ${\textbf{\color{red}82.15}_{\color{red}(+10.90)}}$ & ${\textbf{\color{red}56.20}_{\color{red}(+18.48)}}$ & \textbf{\color{red}21.20} & ${\textbf{\color{red}46.77}_{\color{red}(+7.86)}}$ \\
      \bottomrule
    \end{tabular}
    }
  \vskip -0.1in
\end{table}

\subsection{Cross-Scenario Robustness Analysis}
\subsubsection{Cross-Scenario Collaborative Robustness} To verify the generalization capability of the SCP collaborative framework, we train the ground segmentation task under three data scenarios, namely pure ground G, air-ground fusion AG, and pure aerial A. Table~\ref{tab:AG_ablation_stacked} shows that the relative gains remain positive across all scenarios, with $G_{co}$ reaching 27.52\%, 29.67\%, and 33.34\%, respectively. This result indicates that the proposed routing mechanism does not overfit a single viewpoint topology and remains effective under diverse scenarios.

\subsubsection{Cross-View Complementarity as a Driver of Performance} Experiments further show that the collaborative gain peaks in the Aerial-Only (A) setting, where the segmentation improvement rises from 15.47\% in G and 16.38\% in AG to 18.48\% in A. This trend indicates that the Ground-Only (G) setting mainly provides the same-view local assistance, whereas the joint Air-Ground (AG) setting introduces blended contextual cues. By contrast, the Aerial-Only (A) setting provides distinctive global constraints that are more beneficial for downstream ground parsing. Consequently, after task-conditioned routing, the complementarity between global aerial cues and local ground details can be effectively transformed into auxiliary supervision for fine-grained ground understanding, thereby improving overall model performance.

\begin{table}[H]
  \caption{Comparison of collaborative perception benefits across different collaboration scenarios. Scenarios A, G, and AG denote interaction using aerial, ground, and fused air-ground detection data with segmentation data, respectively.}
  \label{tab:AG_ablation_stacked}
  \centering
  \footnotesize
  \setlength{\tabcolsep}{2.5pt}
  \begin{tabular*}{\linewidth}{@{\extracolsep{\fill}}llccc@{}}
    \toprule
    Scenario & Method & Det (\%) $\uparrow$ & Seg (\%) $\uparrow$ & \boldmath$G_{co}$ (\%) $\uparrow$ \\
    \midrule
    \multirow{2}{*}{G}
    & Vanilla & 71.25 & 37.72 & -- \\
    & SCP & ${\textbf{\color{red}82.15}_{\color{red}(+10.90)}}$ & ${\textbf{\color{red}53.19}_{\color{red}(+15.47)}}$ & $\textbf{\color{red}27.52}$ \\
    \midrule
    
    \multirow{2}{*}{AG}
    & Vanilla & 74.43 & 37.72 & -- \\
    & SCP & ${\textbf{\color{red}86.28}_{\color{red}(+11.85)}}$ & ${\textbf{\color{red}54.10}_{\color{red}(+16.38)}}$ & $\textbf{\color{red}29.67}$ \\
    \midrule
    
    \multirow{2}{*}{A}
    & Vanilla & 76.12 & 37.72 & -- \\
    & SCP & ${\textbf{\color{red}89.59}_{\color{red}(+13.47)}}$ & ${\textbf{\color{red}56.20}_{\color{red}(+18.48)}}$ & $\textbf{\color{red}33.34}$ \\
    \bottomrule
  \end{tabular*}
  \vskip -0.1in
\end{table}


      
      



\subsection{Cross-Altitude Domain Gap Analysis}
Table~\ref{tab:ReID_altitudes} shows that changes in UAV altitude pose a major challenge for air-ground association. As the altitude increases from 10m to 15m, Rank-1 accuracy falls from 36.00\% to 15.90\%, while mAP decreases from 33.60\% to 19.90\%. This drop suggests that aerial identity features are easily affected by altitude. At higher altitudes, targets occupy fewer pixels, local appearance cues become less clear, and viewpoint-induced geometric changes become more pronounced, making direct cross-view matching less reliable. The result also clarifies the roles of the two views in our pipeline. Aerial observations are useful for global localization and candidate search, but their identity cues alone are not stable enough under altitude changes. Ground-view features, in contrast, retain finer local structures and more consistent identity information, and can therefore provide auxiliary evidence for feature alignment. This finding is in line with the progressive design of SCP, where cross-view association links aerial global localization with ground-based fine-grained parsing. 


\begin{table}[H]
  \caption{ReID performance across datasets with varying UAV altitudes.}
  \label{tab:ReID_altitudes}
  \centering
  \footnotesize
    \begin{tabular*}{\linewidth}{@{\extracolsep{\fill}}lcccc@{}}
      \toprule
      Height & \multicolumn{1}{c}{Rank-1 (\%) $\uparrow$} & \multicolumn{1}{c}{Rank-5 (\%) $\uparrow$} & \multicolumn{1}{c}{Rank-10 (\%) $\uparrow$} & \multicolumn{1}{c}{mAP (\%) $\uparrow$} \\
      \midrule
      10m & \textbf{\color{red}36.00} & \textbf{\color{red}47.00} & \textbf{\color{red}51.00} & \textbf{\color{red}33.60} \\
      15m & 15.90 & 28.40 & 36.40 & 19.90 \\
      all & 21.20 & 31.20 & 39.70 & 22.10 \\
      \bottomrule
    \end{tabular*}
  \vskip -0.1in
\end{table}

\section{Discussion}

\textbf{Unified Architectural Synergy:} As summarized in Table \ref{tab:paradigm}, existing paradigms fail to simultaneously satisfy cross-task interaction, cross-view synergy, and practical applicability in real environments. SCP addresses these limitations by synthesizing detection, ReID, and segmentation into a unified ecosystem. Unlike standard MTL restricted to intra-view sharing or 2D/3D AGPS, confined to single-task cooperation, SCP leverages the DLR to enable adaptive cross-task knowledge transfer, achieving a more comprehensive coverage under the considered setting.

\textbf{Coarse-to-Fine Inference:} The proposed pipeline unlocks complex long-horizon tasks by chaining aerial global localization with ground-based fine-grained parsing. This design allows the aerial model to serve as a global observer to overcome occlusions, while the ground model acts as a local executor for pixel-level refinement. This progressive collaboration not only improves performance through complementary exploitation but also offers a robust solution for real-world air-ground applications, such as disaster response.
\begin{table}[H]
  \caption{Comparison of different paradigms. 2D AGPS denotes 2D Air-Ground Perception Systems, and 3D AGPS denotes 3D Air-Ground Perception Systems.}
  \label{tab:paradigm}
  \centering
  \footnotesize
  \begin{tabular*}{\linewidth}{@{\extracolsep{\fill}}lccc@{}}
    \toprule
    Paradigm & Cross-task & Cross-view & Applicability \\
    \midrule
    MTL & $\checkmark$ & $\times$ & $\checkmark$ \\
    2D AGPS & $\times$ & $\checkmark$ & $\checkmark$ \\
    3D AGPS & $\times$ & $\checkmark$ & $\times$ \\
    \midrule
    SCP & $\checkmark$ & $\checkmark$ & $\checkmark$ \\
    \bottomrule
  \end{tabular*}
  \vskip -0.1in
\end{table}


\section{Conclusion}

This paper studies air-ground collaborative perception as a progressive cross-task problem under heterogeneous views. We introduce AGPC, a synchronized UAV-UGV benchmark that connects detection, re-identification, and semantic segmentation within the same video corpus, enabling systematic evaluation of cross-view and cross-task collaboration. Based on AGPC, we propose SCP, a coarse-to-fine framework that organizes perception from aerial global localization to ground target association and identity-aware parsing. Its core module, DLR, selectively routes multi-scale auxiliary features and modulates task-specific information to transfer complementary cues while reducing negative interference. Experiments show that SCP improves both task-level performance and system-level collaborative perception over adapted interaction baselines. More importantly, the results suggest that effective collaboration benefits from stage-aware interaction rather than indiscriminate feature sharing, and that aerial cues are most useful when selectively used to guide fine-grained ground understanding. Future work will explore stronger cross-view association, temporal reasoning, and scalable multi-agent perception.

\bibliographystyle{IEEEtran}
\bibliography{paper}

\end{document}